\documentclass[10pt,letterpaper]{article}
\usepackage{amssymb,amsthm,amsmath}
\usepackage[style=numeric, backend=biber]{biblatex}
\theoremstyle{definition}
\newtheorem{theorem}{Theorem}
\newtheorem{corollary}[theorem]{Corollary}
\newtheorem*{proposition}{Proposition}
\newtheorem*{hypothesis}{Hypothesis}
\newcommand{\roc}{\operatorname{ROC}_f}
\renewcommand{\b}{\text{b}}
\renewcommand{\P}{\mathbb{P}}

\title{A Review of the Receiver Operating Characteristic Curve and a Proof About the Area Beneath It}
\author{Steven Redolfi\footnote{stevenredolfi@gmail.com}}
\begin{document}
\maketitle

\section{Introduction}
\label{intro}

The below Proposition is the subject of this paper. In Theorem \ref{formal statement}, Section \ref{main}, the Proposition is given a formal interpretation and proven. The prerequisite mathematical knowledge, reviewed in the appendix, can be found in any graduate text in mathematical analysis, e.g., Hewitt and Stromberg \cite{hewittstromberg} or Ghatasheh, Redolfi, and Weikard \cite{GRW}.

\begin{proposition}
	The area beneath the Receiver Operating Characteristic (ROC) curve generated by a binary classifier, under a suitable hypothesis, is the same as the probability that the classifier will be able to distinguish between a random positive observation and a random negative observation.
\end{proposition} 

The hypothesis required to ensure these quantities are equal is as follows.

\begin{hypothesis}
The classifier will assign no positive observation the same value as a negative observation.
\end{hypothesis}

If the hypothesis is broken, the area beneath the ROC curve \emph{is no longer} the probability mentioned in the Proposition. However, Corollary \ref{bound} in Section \ref{main} provides a bound on the difference between the area and the probability, and that this difference can be at most 1/2.

A literature review found the earliest claim of the Proposition in Green and Swets \cite{greenswets}, whose assumptions seem close to one of their references, Perterson, Birdsall and Fox \cite{petersonbirdsallfox}. The argument contained in \cite{greenswets} with context given in \cite{petersonbirdsallfox} is recounted in Section \ref{history}. In \cite{greenswets} different hypotheses are made, but an analogous claim is argued, somewhat informally. The argument set forth by this paper is contained in Section \ref{main} along with a simple counterexample when the Hypothesis is not met.

Assuming the Proposition is true, the area beneath an ROC curve is a useful one-number summary of the performance of a binary classifier. Simply put, a binary classifier takes in an observation and gives a number, usually between zero and one, indicating the classifier's confidence that the observation has a property or does not have a property. With this interpretation, if we denote the set of all our finitely many observations by $\Omega$, then the model can be represented by some function $f$ from the set $\Omega$ to the closed interval $[0,1]$, and the set of all observations with the property may be denoted $P$. We assume throughout this document that $\emptyset \neq P \neq \Omega$. Useful in the following is the uniform measure, $\mu$, defined on the finite set $\Omega$. That is, if $E\subset \Omega$, then $\mu(E)$ is the number of elements of $E$ divided by the number of elements in $\Omega$. For instance, $\mu(\Omega)=1$ and $\mu(\emptyset)=0$. One should relate the number $\mu(E)$ to the probability, given a random element $\omega$ in $ \Omega$, that $\omega$ lies in $ E$.

\section{Quantities From Confusion Tables}
\label{confusion table}

Given a threshold $\tau\in\mathbb{R}$, we may presume that, if $\omega\in\Omega$ is an observation, $f(\omega)\geq \tau$ indicates our function ``believes" the observation $\omega$ has the property. If $\omega\in P$, then this may be referred to as a \textbf{true positive}. If $\omega\notin P$, then this may be referred to as a \textbf{false positive}. Similarly, an element $\omega$ with $f(\omega)<\tau$ and $\omega\in P$ will be referred to as a \textbf{false negative} and with $\omega\notin P$ a \textbf{true negative}. One should note the dependence of this property on the function $f$ itself as well as the parameter $\tau$.

Motivated by these observations, we denote the \textbf{true positive rate} of the function $f$ by the function $T_f:\mathbb{R}\to[0,1]$ defined by 

\[T_f(\tau)=\frac{\mu(f^{-1}([\tau,1])\cap P)}{\mu(P)}=\mu(f^{-1}([\tau,1])|P).\]

The \textbf{false positive rate} is then a function $F_f:\mathbb{R}\to[0,1]$ defined by

\[F_f(\tau)=\frac{\mu(f^{-1}([\tau,1])\cap P^c)}{\mu(P^c)}=\mu(f^{-1}([\tau,1])|P^c).\]

One immediately proves that these functions are left-continuous, non-increasing, $T_f(t)=F_f(t)=1$ for all $t\leq 0$, $T_f(t)=F_f(t)=0$ for all $t>1$, and that $T_f$, $F_f$ have jump discontinuities on $f(P)$, $f(P^c)$ respectively.

Using the notation from the appendix, $T_f(t)=f_*(\mu_P)([t,1])$ and $F_f(t)=f_*(\mu_{P^c})([t,1])$ so that $d(-T_f)=f_*(\mu_P)$ and $d(-F_f)=f_*(\mu_{P^c})$.

\section{The ROC curve and its Corresponding Area} 
\label{main}

The function $\roc:\mathbb{R}\to[0,1]^2$ is then defined by $\roc(t)=(F_f(t),T_f(t))$. At the moment, this is not a curve in the mathematical sense, but rather $\roc(\mathbb{R})$ is a set of points in the unit square. One may associate with these points a curve by the following method: there exist finitely many real numbers $t_1,...,t_n$ such that $t_k<t_{k+1}$ for $k=1,...,n-1$ and $\roc(\mathbb{R})=\{\roc(t_k):k=1,...,n\}$. For each $k$, connect the point $\roc(t_k)$ with $\roc(t_{k+1})$ with a line segment. The curve thus generated is the \textbf{receiver operating characteristic} curve. The area beneath this curve, $A$, may be found utilizing (possibly degenerate) trapezoids as

\[A=\sum_{k=1}^{n-1}\frac{T_f(t_k)+T_f(t_{k+1})}{2}(F_f(t_k)-F_f(t_{k+1})).\]

Let us take an alternative view of $A$. Define $\tau_k =\sup \roc^{-1}(\{\roc(t_k)\})$ for $k=1,...,n-1$. By the left-continuity of $\roc$, $\roc(\tau_k)=\roc(t_k)$ and $\roc^+(\tau_k)=\roc(t_{k+1})$. We may thus rewrite the area as

\[A=\sum_{k=1}^{n-1}\frac{T_f(\tau_k)+T_f^+(\tau_k)}{2}(F_f(\tau_k)-F_f^+(\tau_k))=\int T_f^{\b}d(-F_f).\]

We have thus proven the following.

\begin{theorem}
	\label{area underneath the curve}
	The area beneath the ROC curve is exactly the Lebesgue-Stieltjes integral $\int T_f^\b d(-F_f)$.
\end{theorem}

To talk about pairs of observations $(\omega_1,\omega_2),$ we must consider the probability space $\Omega\times\Omega$ with the corresponding product measure $\mu\otimes \mu$. In this case, the product measure is the same as the uniform probability on the set $\Omega\times\Omega$.

The Proposition may be reformulated as ``The area under the ROC curve is the same as the probability that, given an $x\in P$ and a $y\in P^c$, $f(x)>f(y)$." 

If we define \[E=\{(\omega_1,\omega_2)\in\Omega^2:f(\omega_1)>f(\omega_2)\},\] then the probability that, given $x\in P$ and $y\in P^c$, $f(x)>f(y)$ is nothing but the quantity $(\mu\otimes\mu)(E|P\times P^c)$. We come to the reformulation of the Proposition and the corresponding Hypothesis:

\begin{theorem}
	\label{formal statement}
	With all notation as given above, if $f(P)\cap f(P^c)=\emptyset$, then
	\[\int T_f^{\b}d(-F_f)=(\mu\otimes \mu)(E|P\times P^c).\]
\end{theorem}
\begin{proof}
Let us rewrite

\[(\mu\otimes \mu)(E|P\times P^c)=\frac{1}{\mu(P)\mu(P^c)}\int\mu\left((E\cap (P\times P^c))^y\right)\mu(y),\] where the equality utilizes the definition of the probability of $E $ given $P\times P^c$ and equation (\ref{tonelli lemma}) in the appendix. Now,

\[(E\cap(P\times P^c))^y=\begin{cases}
	\emptyset & y\notin P^c\\
	E^y\cap P & y\in P^c
\end{cases}\]
so that the last expression in the previous equation is nothing but

\[\frac{1}{\mu(P)\mu(P^c)}\int_{P^c}\mu(E^y\cap P)\mu(y).\]

Observe that $E^y\cap P=f^{-1}((f(y),1])\cap P$ so the previous integral may be rewritten as

\[\frac{1}{\mu(P^c)}\int_{P^c} T^+(f(y))\mu(y)=\int T^+_f f_*(\mu_{ P^c}),\] where the equality here uses Equations (\ref{substitution}) and  (\ref{given probability}) in the appendix. As noted in Section \ref{confusion table}, $f_*(\mu_{\Omega\setminus P})=d(-F_f)$ so that in total we have shown

\[(\mu\otimes \mu)(E|P\times(\Omega\setminus P))=\int T_f^+ d(-F_f).\]

Under the assumption that $f(P)\cap f(P^c)=\emptyset,$ the equality $T_f^{\b}=T_f^+$ is valid $d(-F_f)$-almost everywhere, whence the theorem is proven.
\end{proof}

Let us introduce a minimum working example to show that we cannot dispense with the Hypothesis, $f(P)\cap f(P^c)=\emptyset$.

Let $f(1)=f(2)=c\in(0,1)$, $\Omega=\{1,2\}$, and $P=\{1\}$. Then, the probability that $f(x)>f(y)$ when $x\in P$ and $y\notin P$ is zero, however, $\roc(0)=(1,1)$ and $\roc(1)=(0,0)$. So, $\roc(\mathbb{R})=\{(0,0),(1,1)\}$ implying that the area beneath the ROC curve is $1/2$. This does not contradict Theorem \ref{formal statement}, since $f(P)\cap f(\Omega\setminus P)=\{c\}\neq\emptyset$.

A simple bound may be placed on the distance between the area beneath the ROC curve and the probability that your binary classifier will be able to differentiate a random positive and random negative observation, given in the following corollary.

\begin{corollary}
	\label{bound}
	If $A$ is the area beneath the ROC curve generated by a binary classifier $f$, and $P$ is the probability that $f$ will differentiate a random positive observation and a random negative observation, and we denote $B=f^{-1}(f(P)\cap f(P^c))$, then
	
	\[0\leq A-P\leq \frac{1}{4}(\mu(B|P)+\mu(B|P^c)).\] 
\end{corollary}
\begin{proof}
	As was shown in Theorem \ref{area underneath the curve}, $A=\int T_f^{\text{b}}d(-F_f)$. An intermediate step in Theorem \ref{formal statement} shows $P=\int T^+_fd(-F_f)$. Note then that
	
	\[A-P=\int(T_f^{\text{b}}-T_f^+)d(-F_f).\]
	
	Since $T_f$ is nonincreasing, $T^{\text{b}}_f-T_f^+\geq 0$ so that, since $d(-F_f)$ is a positive measure, $0\leq A-P$.
	
	As mentioned in Section \ref{confusion table}, $d(-F_f)$ is supported on $f(P^c)$, and $T_f^{\text{b}}\neq T_f^+$ exactly on $f(P)$, so that
	
	\[\begin{split}A-P&=\int_{f(P)\cap f(P^c)}(T_f^{\text{b}}-T_f^+)d(-F_f)\\&=\frac{1}{2}\sum_{\tau\in f(P)\cap f(P^c)}\mu(f^{-1}(\{\tau\})|P)\mu(f^{-1}(\{\tau\})|P^c).\end{split}\]
	
	For any constants $a,b\in [0,1]$, $ab\leq\frac{1}{2}(a^2+b^2)\leq \frac{1}{2}(a+b)$, so	\[A-P\leq\frac{1}{4}(\mu(B|P)+\mu(B|P^c)).\]
\end{proof}

As the previous example showed, it is possible that $A=1/2$, $P=0$, and $B=\Omega$ so that equality holds in the inequality proven above. In any situation, we have shown that $0\leq A-P\leq 1/2$. 

\section{Historical Context for the Proposition}
\label{history}

Here we recount the arguments contained in \cite{petersonbirdsallfox,greenswets} in the most general framing possible.

The ROC curve in \cite{petersonbirdsallfox}\footnote{Introduced in Section 2.6 of that paper.} comes from assuming two functions  $f_{P},f_{P^c}:\mathbb{R}^n\to[0,\infty)$ exist,\footnote{Introduced in Section 2.3 of that paper.} with $P\subset\mathbb{R}^n$ and $m$ the Lebesgue measure, such that for all events $E\subset\mathbb{R}^n$, \[\mathbb{P}_P(E)=\int_Ef_Pm\text{ and }\mathbb{P}_{P^c}(E)=\int_E f_{P^c}m.\] They interpret $\P_P,\P_{P^c}$ as the conditional probabilities that an event is drawn from $P,P^c$ respectively. They further require in their work that $f_{P^c}>0$, so that 

\[\P_P\ll \P_{P^c}\ll m \text{ and } m\ll \P_{P^c}.\]

Note that the chain rule utilizing the Radon-Nikodym theorem implies

\[f_P=(\P_P/m)=(\P_P/\P_{P^c})(\P_{P^c}/m)=(\P_P/\P_{P^c}) f_{P^c},\] whence $(\P_P/\P_{P^c})=f_P/f_{P^c}$, which they refer to as the \textbf{likelihood} function. Their ROC curve comes in when considering the Neyman-Pearson Criteria for accepting that some events $E\subset\mathbb{R}^n$ come from the population $P$. They show that accepting events based upon the Neyman-Pearson criterion is analogous to, after specifying a threshold $\beta\in(0,\infty)$ and accepting all events of the form

\[A(\beta)=\{\omega \in \mathbb{R}^n:(\P_P/\P_{P^c})(\omega)\geq\beta\},\]
generating the  false positive rate and true positive rate $x,y:[0,\infty)\to[0,1]$ defined by $x(t)=\P_{P^c}(A(t))$ and $y(t)=\P_{P}(A(t))$, respectively. Hence, their ROC curve is traced out by $t\mapsto (x(t),y(t))$. They make the further assumption that $x,y$ are differentiable and through considering weighted combination criteria and Neyman-Pearson Criteria, they conclude the slope of the ROC curve at $(x(t),y(t))$ must be $y'(t)/x'(t)$. The document \cite{petersonbirdsallfox} does not investigate the ROC curve further in ways relevant to this paper.

The text \cite{greenswets} concerns itself with the ROC curve, and makes claims about it in both the discrete and absolutely continuous case. The only arguments contained in the text pertain to the absolutely continuous case and seem informal. Their argument is interpreted as follows,\footnote{See Section 2.6 of their text, utilizing their conception of the ROC curve from 2.2.2.} utilizing the notation above.

They vary a \textbf{decision rule}, $\beta\in (0,\infty)$, and if $(\P_P/\P_{P^c})(\omega)>\beta$ then $\omega$ is considered an element of $P$; thus, they collect all observations
\[A(\beta)=\{\omega\in\mathbb{R}:(\P_P/\P_{P^c})(\omega)>\beta\}\]
 For each decision rule $\beta$, one generates the point $(\P_{P^c}(A(\beta)),\P_{P}(A(\beta)))$ to trace out the ROC curve, much as in \cite{petersonbirdsallfox}. If, for each $\beta\in(0,\infty)$, we denote the point on the ROC curve above by $(x(\beta),y(\beta))$, then we're still assuming $x,y$ are strictly decreasing in $\beta$ from $1$ at $\beta=0$ to $0$ as $\beta\to\infty$ in an absolutely continuous fashion. We know the area beneath the ROC curve with this interpretation is 
 \[\int_0^\infty y(\beta)(-x'(\beta))m(\beta)=-\int_{(0,\infty)} y(t) dx(t)=\int_{(0,\infty)}(1-x(t))d(-y)(t),\] where the first equality is since $x'=(dx/m)$ and the second equality is integration by parts and $dy((0,\infty))=-1$. Now, if we set $g=(\P_P/\P_{P^c})$ note that 
 
 \[x(t)=g_*\P_{P^c}((t,\infty))\text{ and }y(t)=g_*\P_{P}((t,\infty)),\] so
 
 \[1-x(t)=g_*\P_{P^c}((0,t))\text{ and }d(-y)=g_*\P_P\] and hence
 
 \[\int_0^\infty y(\beta)(-x'(\beta))m(\beta)=\int_{(0,\infty)}g_*\P_{P^c}((0,t))(g_*\P_P)(t)=\int_{\mathbb{R}}g_*\P_{P^c}((0,g(r)))\P_P(r),\] or
 
 \[(\P_{P^c}\otimes\P_{P})(\{(s,r)\in \mathbb{R}^2:g(s)<g(r)\}).\] This seems to be an absolutely continuous analogue for the Proposition, though the assumptions about posterior distributions $f_P,f_{P^c}$ are unnecessary in our argument from Section 4.
 
 One may note that in the experiments recorded in \cite{greenswets,petersonbirdsallfox}, it appears as though some assumption is made on the probability distribution functions generating the positive and negative classes, e.g., in \cite{petersonbirdsallfox} they assume $f_P,f_{P^c}$ are Gaussian and in \cite{greenswets} they assume Gaussian or exponential at various times, among other (absolutely) continuous distributions.
 
To ensure sufficient thought is given to the historical assumptions, note that not only is an assumption made that $\P_P$ and $\P_{P^c}$ are absolutely continuous with respect to the Lebsgue measure, but the assumption that the ROC curve is continuous, let alone differentiable, is an additional assumption. For instance, suppose that

\[\P_{P^c}(E)=\int_{E}\frac{1}{2}e^{-|t|}dt\]

and, after setting $f_P(t)=\begin{cases}
	1& |t|<\epsilon\\\frac{(1-2\epsilon)e^{\epsilon}}{2} e^{-|t|}&|t|\geq \epsilon
\end{cases},$

\[\P_P(E)=\int_Ef_P(t)dt\] with $\epsilon\in(0,1/2)$. Then

\[(\P_P/\P_{P^c})(t)=\begin{cases}
2e^{|t|}&|t|<\epsilon\\(1-2\epsilon)e^{\epsilon}&|t|\geq \epsilon.
\end{cases}\] Note that this state of affairs satisfies the hypothesis that $\P_P,\P_{P^c}$ are absolutely continuous with respect to Lebesgue measure, however

\[x(t)=\P_{P^c}((\P_P/\P_{P^c})^{-1}((t,\infty)))\] satisfies $x^-((1-2\epsilon)e^\epsilon)=1$, and $x^+((1-2\epsilon)e^\epsilon)=(1/2)\int_{-\epsilon}^\epsilon e^{-|t|}dt=1-e^{-\epsilon}$. 

With that in mind, the assumptions that benefited practitioners when their observations were distributed as assumed in \cite{greenswets,petersonbirdsallfox} may no longer be valid for the data scientist who wishes to make use of the area under the ROC curve which is generated by some arbitrary binary classifier.

\section*{Appendix}

Here, some mathematical notation is explicitly described and common results from measure theory given. The measure theory may be found in \cite{hewittstromberg, GRW}.

If $g$ is a function from the set $X$ to the set $Y$, we denote this $g:X\to Y$. If $E\subset Y$, then $g^{-1}(E)$ is the set of all elements $x$ from $X$ such that $g(x)\in E$, i.e.,

\[g^{-1}(E)=\{x\in X:g(x)\in E\}.\]

Similarly, if $E\subset X$, then $g(E)$ is the set of all elements $y$ of $Y$ such that, for some $x\in E$, $g(x)=y$, i.e.,

\[g(E)=\{y\in Y:g(x)=y\text{ for some }x\in E\}.\]

If $X$ is a set and $E$ another set, the elements from $X$ which are not in $E$ will be denoted $X\setminus E$, i.e., \[X\setminus E=\{x\in X: x\notin E\}.\] If $X$ is assumed to be the universe of the discourse, then $X\setminus E$ may be denoted $E^c$.

If $E\subset\mathbb{R}$ is nonempty, then the supremum of $E$, a real number or $\infty$, will be denoted $\sup E$.

If $E\subset X\times Y$, a horizontal section of $E$ at $y$ is defined as

\[E^y=\{x\in X:(x,y)\in E\}.\]

If $g$ is a real-valued function defined on the positive measure space $(X,\Sigma,\mu)$, then the integral of $g$ over the set $E\in \Sigma$ with respect to the measure $\mu$ will be denoted $\int_E g\mu$. If $\mu$ is a probability, i.e., $\mu(X)=1$, then the probability of $A\in \Sigma$ given $B\in \Sigma$ will be denoted $\mu(A|B)$ and defined by $\mu(A\cap B)/\mu(B)$.

Given the $g$ above, we will denote by $g_*\mu$ the measure defined, on all subsets $E$ of $\mathbb{R}$ such that $g^{-1}(E)$ is $\mu$ measurable, by $(g_*\mu)(E)=\mu(g^{-1}(E))$.  Finally, if $V\in \Sigma$ and $0<\mu(V)<\infty$, then $\mu_V$ will be the measure defined, on all sets $E$ in $\Sigma$, by $\mu_V(E)=\mu(E\cap V)/\mu(V)$. If $(Y,S,\nu)$ is another positive measure space, we will denote the product measure by $\mu\otimes \nu$. A fact we shall have use of is that, if $E$ is $\mu\otimes\nu$-measurable and $(\mu\otimes \nu)(X\times Y)<\infty$, then
\begin{equation}
	\label{tonelli lemma}
	(\mu\otimes \nu)(E)=\int_Y\mu(E^y)\nu(y).
\end{equation}

If $g:\mathbb{R}\to\mathbb{R}$ is a function, we will write

\[g^+(x)=\lim_{t\downarrow x}g(t)\text{ and }g^-(x)=\lim_{t\uparrow x}g(t)\] whenever these limits exist. If $g$ is a function where these limits exist for every $x\in\mathbb{R}$, the balanced version of the function will be denoted $g^{\text{b}}:\mathbb{R}\to\mathbb{R}$ and is defined by

\[g^{\text{b}}(x)=\frac{g^+(x)+g^-(x)}{2}.\]

If the function $g$ is of bounded variation, then $dg$ will denote the Lebesgue-Stieltjes measure generated by $g$. Recall that if $a,b\in\mathbb{R}$ and $a<b$, that

\[dg(\{a\})=g^+(a)-g^-(a)\text{ and }dg((a,b))=g^-(b)-g^+(a).\]

If $g$ is nonnegative and $\mu$-measurable, then

\begin{equation}
	\label{given probability}
	\int g\mu_V =\frac{1}{\mu(V)}\int_V g\mu;
\end{equation}
furthermore, if $h:\mathbb{R}\to\mathbb{R}$ is $g_*\mu$ measurable and integrable, then $h\circ g$ is $\mu$-measurable and integrable, and
\begin{equation}
	\label{substitution}
	\int h(g_*\mu)=\int (h\circ g) \mu.
\end{equation}

We will have need of a portion of Tonelli's theorem: If $h$ is a nonnegative function defined on $X\times Y$ and measurable with respect to $\mu\otimes \nu$, then \[\int\left(\int h(x,y)\mu(x)\right)\nu(y)=\int\left(\int h(x,y)\nu(y)\right)\mu(x).\]

We'll also have use of the concept of absolute continuity: given two  positive measures $\mu$ and $\nu$ defined on the same $\sigma$-algebra, $\mu$ is absolutely continuous with respect to $\nu$, written $\mu\ll\nu$, if for all events $E$, $\nu(E)=0$ implies $\mu(E)=0$.

Finally, the Radon-Nikodym theorem states that if $\mu\ll \nu$ and each measure is $\sigma$-finite, then there is a function, denoted by $(\mu/\nu)$, such that 

\[\mu(E)=\int_E (\mu/\nu)\nu\]

for each event $E$.

\newpage
\printbibliography

\end{document}